\documentclass{article}

\usepackage{gensymb}
\usepackage{microtype}
\usepackage{graphicx}
\usepackage{subfigure}
\usepackage{booktabs} 
\usepackage{amsmath,amssymb, amsthm, bm}
\usepackage{color,soul}
\DeclareMathOperator*{\argmax}{argmax}
\usepackage{hyperref}

\usepackage{hyperref}

\newtheorem{theorem}{Theorem}

\newtheorem{proposition}[theorem]{Proposition}
\newtheorem{lemma}[theorem]{Lemma}

\usepackage[accepted]{icml2019arXiv}

\icmltitlerunning{FPCA for Extrapolating Multi-stream Longitudinal Data}

\begin{document}

\twocolumn[
\icmltitle{Functional Principal Component Analysis for  Extrapolating Multi-stream Longitudinal Data}

\icmlsetsymbol{equal}{*}

\begin{icmlauthorlist}
\icmlauthor{Seokhyun Chung}{to}
\icmlauthor{Raed Kontar}{to}
\end{icmlauthorlist}

\icmlaffiliation{to}{Industrial \& Operations Engineering, University of Michigan, Ann Arbor, MI, USA}

\icmlcorrespondingauthor{Raed Kontar}{alkontar@umich.edu}

\icmlkeywords{Functional principal component analysis, Gaussian process}

\vskip 0.3in
]

\printAffiliationsAndNotice{}

\begin{abstract}

The advance of modern sensor technologies enables collection of multi-stream longitudinal data where multiple signals from different units are collected in real-time. In this article, we present a non-parametric approach to predict the evolution of multi-stream longitudinal data for an in-service unit through borrowing strength from other historical units. Our approach first decomposes each stream into a linear combination of eigenfunctions and their corresponding functional principal component (FPC) scores. A Gaussian process prior for the FPC scores is then established based on a functional semi-metric that measures similarities between streams of historical units and the in-service unit. Finally, an empirical Bayesian updating strategy is derived to update the established prior using real-time stream data obtained from the in-service unit. Experiments on synthetic and real world data show that the proposed framework outperforms state-of-the-art approaches and can effectively account for heterogeneity as well as achieve high predictive accuracy.

\end{abstract}

\section{Introduction}
\label{S:1}

Among various environments where longitudinal data is gathered, the environment covered in this study is a \textit{multi-stream} and \textit{real-time} environment. Recent progress in sensor and data storage technologies has facilitated data collection from multiple sensors in real-time as well as the accumulation of historical signals from multiple similar units during their operational lifetime. This data structure where multiple signals across different units are collected is referred to as multi-stream longitudinal data. Examples include: vital health signals from patients collected through wearable devices \cite{caldara2014wearable, magno2016infinitime}, battery degradation signals from cars on the road \cite{meeker2014reliability, salamati2018experimental} and energy usage patterns from different smart home appliances \cite{hsu2017design}. 

In this article, we  propose an efficient approach to extrapolate multi-stream data for an in-service unit through borrowing strength from other historical units. An illustrative example is provided in Figure \ref{fig:ms}. In this figure, there are $N$ historical units and an in-service unit whose index is denoted by $r$. Each unit has $M$ identical sensors from which each respective signal forms a stream. Multi-stream data from the in-service unit is partially observed up to the current time instance $t^\ast$. Our goal is to extrapolate stream data from the in-service unit $r$ over a future period $t \geq t^* \in \mathcal{T}$ where $\mathcal{T}$ is the time domain of interest.

\begin{figure}[H]
	\begin{center}
		\centerline{\includegraphics[width=\columnwidth]{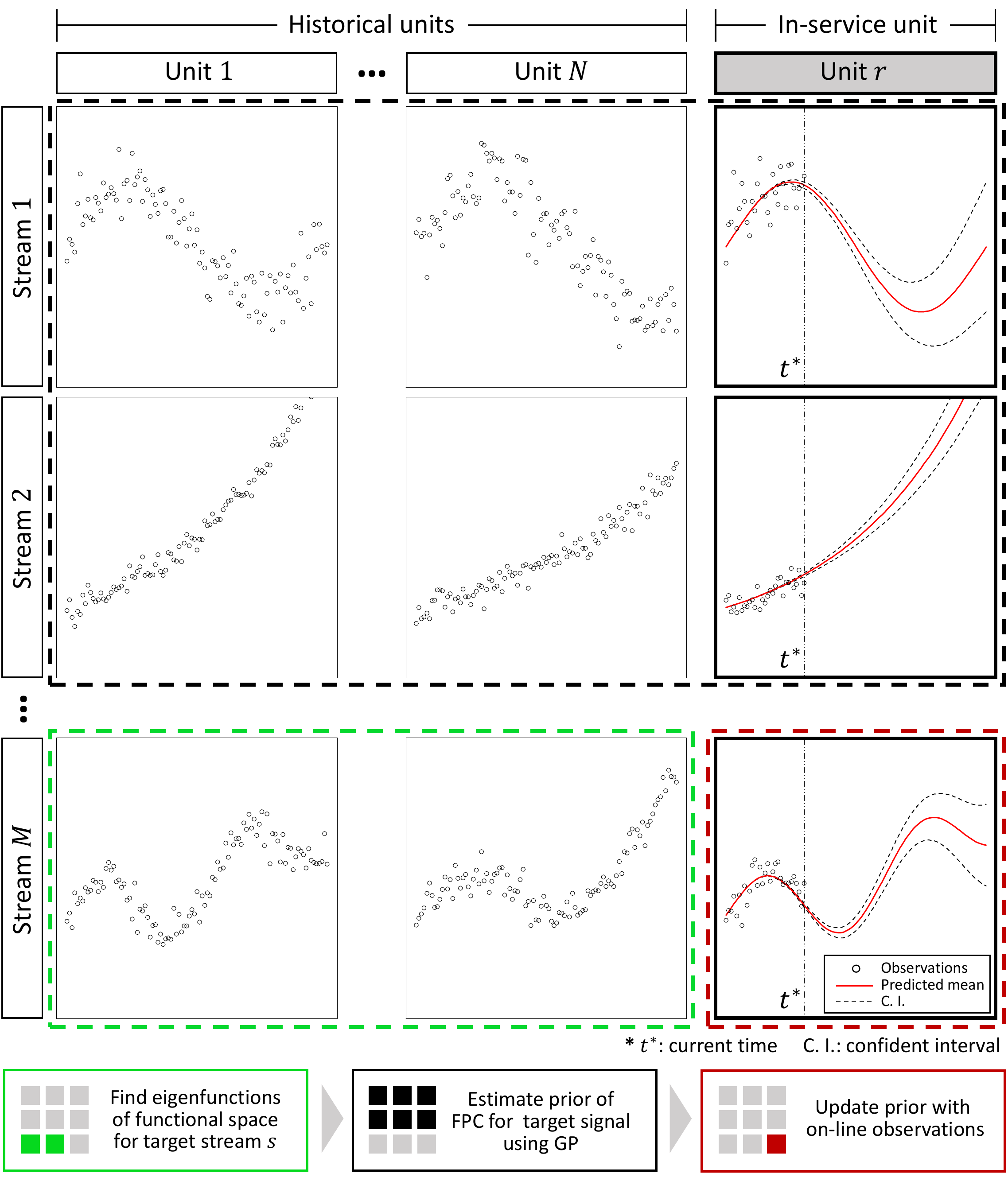}}
		\caption{Extrapolation of multi-stream longitudinal data for an in-service unit.}
		\label{fig:ms}
	\end{center}
	\vskip -0.4in
\end{figure}

In mathematical notation, let $I = \lbrace1,...,N,r\rbrace$ and $I^{\mathcal{H}} = \lbrace1,...,N\rbrace$ be the respective index sets for all available units including the in-service unit $r$ and the units in our historical dataset. For each unit $i \in I$, we have $M$ streams of data where $l \in L = \lbrace 1,...,M\rbrace$. For the $i$th unit, the history of observed data for a specific stream $l$ is denoted as $X^{(l)}_i(t_{iu})$, where $\lbrace t_{iu}: u=1,...,p^{(l)}_i\rbrace \subset \mathcal{T}$ represents the observation time points and $p^{(l)}_i$ represents the number of observations for signal $l$ of unit $i$. The underlying principle of our model is borrowing strength from a sample of curves $\{X_i^{(l)} (t): i \in I^{\mathcal{H}}, l \in L \}$ to predict individual trajectories $X_r^{(l)} (t)$ over a future time period  $t \geq t^* \in \mathcal{T}$. Without loss of generality, throughout the article we focus on predicting stream $s\in L$, which we refer to as the target stream. Note that the target stream in Figure \ref{fig:ms} is stream $M$. To achieve this goal we exploit functional principal component analysis (FPCA), which is a non-parametric tool for functional data analysis \cite{ramsay2nd}. Indeed, FPCA has recently drawn increased attention due to its flexibility, uncertainty quantification capabilities and the ability to handle sparse and irregular data \cite{peng2009geometric,di2009multilevel,wang2016functional,xiao2018fast}. However, advances in FPCA fall short of handling \textit{multi-stream} data and \textit{real-time} predictions.

Our overall framework is summarized at the bottom of Figure \ref{fig:ms}. Specifically, historical signals $X^{(s)}_i(t)$, $i \in I^{\mathcal{H}}$, from the target stream $s$  are decomposed into a linear combination of orthonormal eigenfunctions that form their functional space. The coefficients of the linear combination are called functional principal component (FPC) \textit{scores}. With the assumption that the target signal of the in-service unit $X^{(s)}_r(t)$ lies in the same functional space, a proper estimation of FPC scores associated with $X^{(s)}_r(t)$ is required. Here, we propose to establish a prior on these FPC scores using information from streams $l \in L^{-s} = L \backslash \lbrace s\rbrace$. Specifically, a Gaussian process (GP) prior for FPC scores of $X^{(s)}_r(t)$ is built using a functional semi-metric that measures similarities of streams $L^{-s}$ between historical units and the in-service unit $r$. The underlying principle is that unit $X^{(s)}_r(t)$ will exhibit more commonalities with historical units that exhibit similar trends in streams $L^{-s}$. For example, if stream $s$ denotes degradation trajectories and $L^{-s}$ denotes  external factors such as temperature. Then $X^{(s)}_r(t)$ will share more commonalities with a subset of historical signals $X^{(s)}_i(t)$, $i \in I^{\mathcal{H}}$, degrading under similar external factors (similar temperatures). This approach allows us to address heterogeneity in the data. Lastly, an empirical Bayesian updating strategy is derived to update the established prior using real-time stream data obtained from the in-service unit.

\section{Literature Review}
\label{S:2}

There has been extensive literature on the extrapolation of longitudinal signals under a single stream setting. However, literature has mainly focused on parametric models due to their computational efficiency and ease of implementation \cite{nagi2005residual, gebraeel2008prognostic, si2012remaining, si2013wiener, kontar2017remaining}.  Such models have been applied in healthcare, manufacturing and mobility applications specifically to understand the remaining useful life of operational units. Unfortunately, in real-world applications, parametric modeling is vulnerable to model misspecifications, and if the specified form is far from the truth, predictive results will be misleading. For instance, parametric representations are specifically challenging when data is sparse or when the underlying physical and chemical theories guiding the process are unknown. 

To address this issue, recent attempts at non-parametric approaches have been based on FPCA \cite{zhou2011degradation, zhou2012degradation, fang2015adaptive} or multivariate Gaussian processes \cite{alvarez2011computationally,saul2016chained, kontar2018nonparametric, kontar2019minimizing}. These studies show that such non-parametric approaches outperform parametric models in case where functional forms are complex and exhibit heterogeneity. Nevertheless, the foregoing works have dealt with only single stream cases.

On the other hand, the few literature that addressed multi-stream settings have focused on data fusion approaches. Data fusion in this case refers to aggregating all streams into a single stream using fusion mechanisms. In health related applications, this fused stream is coined as a health-index which is often derived through a weighted combination of the $M$ data streams \cite{Liu2013Data, song2018statistical}. Such methods require regularly sampled observations and enforce strong parametric assumptions. An alternative data fusion approach includes multivariate FPCA \cite{fang2017scalable,fang2017multistream}. However, since data fusion methods are operated by aggregating multi-streams into a single or a smaller group of streams, they are not capable of predicting individual stream trajectories and thus have limited applications.

Compared to current literature, our contribution can be summarized as follows. We propose an FPCA-based model that  provides individualized predictions in a multi-stream environment. Our model is able to automatically account for heterogeneity in the data and screen the sharing of information between the in-service unit and units in our historical dataset. We then derive a computationally efficient Bayesian updating strategy to update predictions when data is collected in real-time. We demonstrate the advantageous features of our approach compared to state-of-the-art methods using both synthetic and real-world data.    

The rest of this paper is structured as follows. In section \ref{S:3}, we briefly revisit the FPCA. In section \ref{S:4}, we discuss our proposed model. Numerical experiments using synthetic data and real-world data are provided in section \ref{S:5}. Finally, section \ref{S:6} discusses the computational complexity of our model. Technical proofs, a detailed code and additional numerical results are available in the \href{https://alkontar.engin.umich.edu/publications/}{supplementary materials}.

\section{Brief Review of FPCA}
\label{S:3}
From an FPCA perspective, longitudinal signals observed in a given time domain can be decomposed into a linear combination of orthonormal basis functions with corresponding FPC scores as coefficients. Therefore, FPCA can be regarded as a dimensionality reduction method in which a signal corresponds to a vector in a functional space defined by the basis functions. The basis functions are referred to as eigenfunctions. Let us assume that the longitudinal signals, over a given time domain $\mathcal{T}$, are generated from a square-integrable stochastic process $X(t)$ with its mean $\mathbb{E}[X(t)] = \mu(t)$ and covariance defined by a positive semi-definite kernel $G(t_1,t_2) = \text{Cov}(X(t_1), X(t_2))$ for $t_1, t_2\in \mathcal{T} $. Using Mercer's theorem on $G(t_1, t_2)$, we have 
\begin{equation*}
    G(t_1, t_2) = \sum_{k=1}^{\infty}\lambda_k \phi_k(t_1) \phi_k(t_2),
\end{equation*}
where $\phi_k(t)$ presents the $k$th eigenfunction of the linear Hilbert-Schmidt operator $G: L^2(\mathcal{T}) \rightarrow L^2(\mathcal{T}), G(f) = \int_\mathcal{T}G(t_1,t_2)f(t_1)dt_1$ ordered by the corresponding eigenvalues $\lambda_k$, $\lambda_1\geq \lambda_2 \geq \dots \geq 0$. The eigenfunctions $\phi_k(t)$ form a set of orthonormal basis in the Hilbert space $L^2(\mathcal{T})$. Following the Karhunen-Lo\'eve decomposition, the centered stochastic process $X(t) - \mu(t)$ can then be expressed as
\begin{equation*}
 X(t) - \mu(t) = \sum_{k=1}^{\infty}\xi_{k}\phi_{k}(t) + \epsilon(t),
 \label{S1:eq1}
\end{equation*}
where $\xi_{k}  = \int_\mathcal{T} \big( X(t) - \mu(t) \big) \phi_k(t) $ presents FPC scores associated with $\phi_k(t)$. The scores are uncorrelated normal random variables with zero-mean and variance $\lambda_k$; that is, $\mathbb{E}[\xi_k] = 0, \forall k \in \mathbb{N}$ and $\mathbb{E}[\xi_{k_1}\xi_{k_2}] = \delta_{k_1k_2}\lambda_{k_1}, \forall k_1,k_2\in \mathbb{N}$ where $\delta_{k_1k_2}$ denotes the Kronecker delta. Also, $\epsilon(t)$ is additive Gaussian noise.
 
 This idea of projecting signals onto a functional space spanned by eigenfunctions was first introduced by \citet{rao1958some} for growth curves in particular. Basic principles \cite{Castro1986principal,besse1986principal} and theoretical characteristics \cite{silverman1996smoothed,boente2000kernel,Kneip2001Inference} were then developed. These ideas were expanded to longitudinal data settings in the seminal work of \citet{yao2005functional}. After that, the FPCA was applied and extended to a wide variety of applications, where multiple works tackled fast and efficient estimation of the underlying covariance surface \cite{huang2008functional, di2009multilevel, peng2009geometric, goldsmith2013corrected, xiao2016fast}.

\section{Extrapolation of Multivariate Longitudinal Data} \label{S:4}

\subsection{FPCA for Signal Approximation}\label{S:4.1}

Now we discuss our proposed non-parametric approach for extrapolation of multi-stream longitudinal data. Hereon, unless there is ambiguity, we suppress subscripting the target stream with $(s)$. Using historical signals $X_i(t)$, $i\in I^{\mathcal{H}}$, we decompose the target stream $s$ as 
\begin{equation}
X_{i}(t) = \mu(t) + f_{i}(t) + \epsilon(t), \quad i \in I^{\mathcal{H}},
\label{eq:3.1}
\end{equation}

where $f_{i}(t)$ represents random effects characterizing stochastic deviations across different historical signals in stream $s$ and $\epsilon(t)$ denotes additive noise.  We assume $f_{i}(t)$ and  $\epsilon(t)$ are independent. Through an FPCA decomposition, we have that $f_i(t)= \sum_{k=1}^{\infty}\xi_{ik}\phi_{k}(t)$. This decomposition is an infinite sum, however, only a small number of eigenvalues are commonly significantly non-zero. For these values the corresponding scores $\xi_{ik}$ will also be approximately zero. Therefore, we approximate this decomposition as $f_i(t)=\sum_{k=1}^{K}\xi_{ik}\phi_{k}(t)$, where $K$ is the number of significantly nonzero eigenvalues.
\begin{equation}
X_{i}(t) = \mu(t) + \sum_{k=1}^{K}\xi_{ik}\phi_{k}(t) + \epsilon(t), \quad i \in I^{\mathcal{H}}.
\label{eq:3.3}
\end{equation}
Here we follow the standard estimation procedures in \citet{di2009multilevel} and \citet{goldsmith2018refund} to estimate the model parameters where  $\mu(t)$ is obtained by local linear smoothers \cite{fan1st}, while $K$ is selected to minimize the modified Akaike criterion. Now given that the in-service unit $r$ lies in the same functional space spanned by $\phi_k(t)$, our task is to find the individual distribution of  $\xi_{rk}$ using the partially observed multi-stream data from unit $r$. Specifically, we aim to find $X_{r}(t) = \mu(t) + \sum_{k=1}^{K}\xi_{rk}\phi_{k}(t) + \epsilon(t)$.

\subsection{Estimation for Prior Distribution of FPC scores via GP}\label{S:4.2}

Next, we estimate the prior distribution of $\xi_{rk}$ based on the key premise that $X_{r}(t)$ will behave more similarly to $X_{i}(t)$ for some units $i\in I^{\mathcal{H}}$ whose signals $X^{(l)}_{i}(t)$ for $l\in L^{-s}$ are similar to the corresponding signals of the in-service unit $X^{(l)}_{r}(t)$. To this end, for $k \in \lbrace 1,...,K\rbrace$, we model a functional relationship between $[\xi_{1k},...,\xi_{Nk}, \xi_{rk}]'$ and $X_i^{(l)}(t),..., X_N^{(l)}(t),X_r^{(l)}(t)$ for $l\in L^{-s}$ as 
\begin{equation}
\begin{split}
 [&\boldsymbol{\xi'}_{k},  \xi_{rk}]' \\
 &= \mathcal{G}_k\big(\bm{X}^{(-s)}_1(t), ..., \bm{X}^{(-s)}_N(t), \bm{X}^{(-s)}_r(t)\big) + \bm{\epsilon}_{\xi_k},
 \label{eq:3.5}
\end{split}
\end{equation}
where $\bm{X}^{(-s)}_i(t)=[X_i^{(l)}(t)]_{l\in L^{-s}} $ for $i\in I$, $\boldsymbol{\xi}_{k} = [\xi_{ik}]_{i\in I^{\mathcal{H}}}$, and $\bm{\epsilon}_{\xi_k} \sim \mathcal{N}(\mathbf{0}, \sigma^2_k)$.

The idea here, is to model $\mathcal{G}_k$ as a GP with a covariance function defined by a similarity measure between the observed signals, i.e., a functional similarity measure. Specifically, for any $k \in \lbrace 1,...,K\rbrace$, the vector of FPC scores $[\boldsymbol{\xi}'_{k}, \xi_{rk}]'$ will follow a multivariate Gaussian distribution    
\begin{equation}
 \begin{bmatrix} \boldsymbol{\xi}_{k} \\ \xi_{rk} \end{bmatrix}  \sim \mathcal{N}\bigg(\mathbf{0}_{N+1}, \begin{bmatrix}\mathbf{C}_k + \sigma_k^2\mathbf{I}_{N} & \mathbf{c}_k \\ \mathbf{c}'_k & c^{(r)}_k \end{bmatrix} \bigg),
 \label{eq:3.6}
\end{equation}
where $\mathbf{C}_k \in \mathbb{R}^{N\times N}$ is constructed such that its $(i,j)$th element is $h(i,j; \bm{\theta}_k)$ for $i,j \in I^\mathcal{H}$, $\mathbf{c}_k = [h(i,r,\bm{\theta}_k)]_{i\in I^{\mathcal{H}}}$, $c^{(r)}_k=h(r,r,\bm{\theta}_k)$, and $h$ denotes a covariance function defined as
\begin{equation*}
 h(i,j;\bm{\theta}_k) = \alpha_k \exp \bigg(-\frac{1}{2}\sum_{l\in L^{-s}} \frac{\big\Vert X^{(l)}_i(t) - X^{(l)}_{j}(t) \big\Vert^2_l}{\big(\beta^{(l)}_{k}\big)^2}  \bigg),
 \label{eq:3.7}
\end{equation*}
in which $\Vert \cdot \Vert_l$ is a semi-metric providing as similarity measure across functions, and $\alpha_k$ and $\beta_{k}^{(l)}$ are hyperparameters for streams $l\in L^{-s}$. For notational simplicity, we introduce $\bm{\theta}_k = \big[\alpha_k, \beta_{k}^{(1)},...,\beta_{k}^{(s-1)}, \beta_{k}^{(s+1)},...,\beta_{k}^{(M)}\big]'$. 

To show the validity of the GP \eqref{eq:3.6}, we provide the following lemma.
\begin{lemma}
The matrix $\begin{bmatrix}\mathbf{C}_k + \sigma_k^2\mathbf{I}_N & \mathbf{c}_k \\ \mathbf{c}'_k & c^{(r)}_k \end{bmatrix}$ corresponding to the covariance function $h(i,j;\bm{\theta}_k)$ a valid covariance matrix.
\end{lemma}
\begin{proof}
See Section A.1 in supplementary document.
\end{proof}

One possible semi-metric that represents the similarity between two signals can be derived based on FPCA. Let $\mathcal{T}_{t^*} \subset \mathcal{T}$ denote the time domain for observations up to $t^*$. Note that we define $\mathcal{T}_{t^*}$ since the signals of the in-service unit $r$ are available only for $t\in \mathcal{T}_{t^*}$. For $i,j\in I$, $l \in L^{-s}$, and $t \in \mathcal{T}_{t^*}$, the semi-metric based on FPCA for two signals $X^{(l)}_i(t)$ and $X^{(l)}_j(t)$ can be represented as 
\begin{align}
\big\Vert &X^{(l)}_i(t)  - X^{(l)}_{j}(t) \big\Vert_l \nonumber \\
&= \sqrt{\sum_{k=1}^{K(l)}\bigg(\int_{\mathcal{T}_{t^*}} \big[ X^{(l)}_i(t) - X^{(l)}_{j}(t) \big] \psi^{(l)}_k(t)dt\bigg)^2 },
\end{align}
where $\psi_k^{(l)}(t)$ is $k$th eigenfunction derived by the FPCA on $X^{(l)}_i(t)$ for $i\in I$ and $t \in \mathcal{T}_{t^*}$, and $K(l)$ is the number of eigenfunctions. We would like to point out that $\int_{\mathcal{T}_{t^*}} \big[X^{(l)}_i(t) - X^{(l)}_{j}(t)\big] \psi^{(l)}_k(t)dt$ is the difference between the FPC scores of $X^{(l)}_i(t)$ and $X^{(l)}_j(t)$ associated with $\psi^{(l)}_k(t)$, which implies that this metric measures the Euclidean distance between two vectors composed of the corresponding FPC scores. 

In order to optimize the hyperparameter $\bm{\Theta}_k = [\bm{\theta}'_k, \sigma_k]'$ for the multivariate Gaussian distribution \eqref{eq:3.6}, we maximize the marginal log-likelihood function of $\boldsymbol{\xi}_{k}=[\xi_{1k},...,\xi_{Nk}]'$ given $\bm{X}^{(-s)}_1(t), ..., \bm{X}^{(-s)}_N(t)$. Let $\mathcal{X}$ denote the observations of signals $X^{(l)}_i(t)$ for  $l\in L^{-s}$ units $i\in I^{\mathcal{H}}$, that is $\mathcal{X} = \big\lbrace X^{(l)}_i(t)| {t\in T_i^{(l)}, l\in L^{-s}, i\in I^{\mathcal{H}}} \big\rbrace$ where $T_i^{(l)}=\lbrace t_{iu} | i\in I^{\mathcal{H}}, l\in L^{(-s)}, u=1,...,p^{(l)}_i \rbrace$. Also, let  $z_{ik}$ denote the true underlying latent values corresponding to the FPC scores $\xi_{ik}$ and let $\bm{z}_k=[z_{1k},...,z_{Nk}]'$. Then the marginal likelihood is given as 
\begin{align*}
    P(\boldsymbol{\xi}_{k}&|\mathcal{X}, \bm{\Theta}_k)\\ 
    & = \int P(\boldsymbol{\xi}_{k} |\bm{z}_k,\mathcal{X}, \bm{\Theta}_k)P(\bm{z}_k|\mathcal{X}, \bm{\Theta}_k)d\bm{z}_k\\
    & = \int \prod_{i\in I^{\mathcal{H}}} P(\xi_{ik}|z_{ik},\mathcal{X}, \bm{\Theta}_k) P(\bm{z}_k|\mathcal{X}, \bm{\Theta}_k )d\bm{z}_k
\end{align*}
where $P(\xi_{ik}|z_{ik}, \mathcal{X}, \bm{\Theta}_k) = \mathcal{N}(0, \sigma_k^2)$ and $P(\bm{z}_k|\mathcal{X}, \bm{\Theta}_k) = \mathcal{N}\big(\boldsymbol{0}_{N}, \mathbf{C}_k\big)$. The second equality follows from the fact that the error is an additive Gaussian noise. Thus, $\prod_{i\in I^{\mathcal{H}}} P(\xi_{ik}|z_{ik},\mathcal{X}, \bm{\Theta}_k) = \mathcal{N}(\boldsymbol{0}_{N}, \sigma_k^2\mathbf{I}_{N})$ and the log-likelihood of $P(\boldsymbol{\xi}_{k}|\mathcal{X}, \bm{\Theta}_k)$ is 
\begin{align*}
\log P(\boldsymbol{\xi}_k | \mathcal{X}, \bm{\Theta}_k) &= -\frac{1}{2} \big\langle \bm{\Xi}_k, (\mathbf{C}_k + \sigma_k^2 \mathbf{I}_N)^{-1} \big\rangle_{tr} \\
     & \qquad \> -\log\big|\mathbf{C}_k + \sigma_k^2 \mathbf{I}_N\big| - \frac{n}{2}\log 2 \pi,   
\end{align*}
where $\langle \bm{A}, \bm{B} \rangle_{tr} = trace(\bm{A}\bm{B})$ and $\bm{\Xi}_k = \boldsymbol{\xi}_k \boldsymbol{\xi}'_k $. 
As a consequence, the optimized hyperparameters denoted by  $\bm{\Theta}_k^* = [\bm{\theta}^{*'}_{k}, \sigma_k^*]'$ are found by maximizing the marginal log-likelihood. More formally, we have
\begin{align*}
    \bm{\Theta}_k^* = [\bm{\theta}^{*'}_{k}, \sigma_k^*]' = \argmax_{\bm{\Theta}_k} \log P(\boldsymbol{\xi}_k | \mathcal{X}, \bm{\Theta}_k).
\end{align*}

Following multivariate normal theory, the posterior predictive distribution of $\xi_{rk}$, given \eqref{eq:3.6} and $\bm{\Theta}_k^*$, is derived as
\begin{equation}\label{eq:3.9}
P(\xi_{rk} | \boldsymbol{\xi}_{k},\mathcal{X}, \bm{\Theta}_k^*) = \mathcal{N}\big(\hat{\xi}_{rk}, \hat{\sigma}_{rk}^2 \big)
\end{equation}
with
\begin{align*}
&\hat{\xi}_{rk} = \mathbf{\hat{c}}_k'(\mathbf{\hat{C}}_k+ (\sigma_k^*)^2\mathbf{I}_{N})^{-1}\boldsymbol{\xi}_{k},\\
&\hat{\sigma}_{rk}^2 = \hat{c}^{(r)}_k  -\mathbf{\hat{c}}_k'(\mathbf{\hat{C}}_k + (\sigma_k^*)^2\mathbf{I}_{N})^{-1}\mathbf{\hat{c}}_k.
\end{align*}

Here we note that for each $k\in \lbrace 1,...,K \rbrace$, we can derive $\hat{\xi}_{rk}$ and $\hat{\sigma}_{rk}^2$ using an independent GP as the FPC scores from different orthonormal basis functions are uncorrelated. This facilitates scalability of computation  as for different $k$ we can derive $\lbrace \hat{\xi}_{rk} \}$ and $\{ \hat{\sigma}_{rk}^2 \}$ in parallel. As shown in the computational complexity derivations in section \ref{S:6}, this aspect is important specifically in a real-time environment where predictions need to be continuously updated. 

Now combining equations \eqref{eq:3.3} and \eqref{eq:3.9}, we obtain the predictive mean $\hat{X}_r(t)$ and variance $\hat{\sigma}_{r}^2(t)$ of ${X}_r(t)$ as follows
\begin{equation}
\begin{split}
& \hat{X}_r(t) =  \mu(t) + \sum_{k=1}^{K}\hat{\xi}_{rk}\phi_{k}(t), \\
& \hat{\sigma}^2_{r}(t) =  \sigma^2 _\mu(t)+ \sum_{k=1}^{K}\hat{\sigma}_{rk}^2 \phi_{k}^{2}(t) + \sigma^2_\epsilon . 
\end{split}
\end{equation}

Here note that $\mu(t)$, $\sigma^2_\mu(t)$ and $\sigma^2_\epsilon$ are model parameters corresponding to the estimated FPCA model in \eqref{eq:3.3}, where $\sigma^2_\epsilon$ denotes the estimated variance of $\epsilon(t)$. Here we  recall that $\hat{X}_r(t) =\hat{X}^{(s)}_r(t)$ as the index $(s)$ is dropped for the target stream. 

\subsection{Empirical Bayesian Updating with On-line Data}\label{S:4.3}

In the previous section, we derive a prior for $\hat{X}_r(t)$ and $\hat{\sigma}_{r}^2(t)$ using data observed from streams $l\in L^{-s}$. Here, we develop an empirical Bayesian approach to update $\hat{X}_r(t)$ and $\hat{\sigma}_{r}^2(t)$ given the target stream ($s$) observations from the in-service unit $r$. Specifically, given the prior distributions $P(\xi_{rk}) = \mathcal{N}\big(\hat{\xi}_{rk}, \hat{\sigma}^2_{rk}\big)$ for each $k$ and given the observations $\bm{X}_r(\bm{t})$ at $\bm{t} = [t_{r1},...,t_{rp_r}]'$, the posterior $P(\xi_{rk} \> | \> \bm{X}_r(\bm{t}))$ is given in Proposition \ref{prop:1}. 


\begin{proposition}
Given that $X_r(t) = \mu(t) + \sum_{k=1}^{K}\xi_{rk}\phi_k(t) + \epsilon(t)$, where the prior distribution of $\xi_{rk}$ is $\mathcal{N}\big(\hat{\xi}_{rk}, \hat{\sigma}_{rk}^2\big)$, and the FPC scores are pairwise independent. Then, the posterior distribution of the FPC scores, such that $\xi^{*}_{rk} = P(\xi_{rk} \> | \> \bm{X}_r(\bm{t}))$, is given as

\begin{equation*}
[\xi^{*}_{r1}, ...,\xi^{*}_{rK}]' \sim \mathcal{N}(\bm{\xi}^{*}, \bm{\Sigma}^{*}) 
\end{equation*}
where 
\begin{align*}
  &\bm{\xi}^{*} = \bm{\Sigma}^{*}\bigg(\bm{\Sigma}_0^{-1}\bm{\mu}_0 + \frac{1}{\sigma_\epsilon^2}\bm{\Phi}(\mathbf{t})'(\bm{X}_r(\mathbf{t})-\bm{\mu}(\mathbf{t})) \bigg),   \\
  &\bm{\Sigma}^{*} = \bigg(\frac{1}{\sigma_\epsilon^2}\bm{\Phi}(\mathbf{t})'\bm{\Phi}(\mathbf{t}) + \bm{\Sigma}_0^{-1} \bigg)^{-1}
\end{align*}
with 
\begin{equation*}
\begin{split}
& \bm{\mu}_0 = \big[\hat{\xi}_{r1},...,\hat{\xi}_{rK} \big]', \\
& \bm{\Sigma}_0 = \text{diag}\big(\hat{\sigma}_{r1}^2, ..., \hat{\sigma}_{rK}^2\big), \\
& \bm{X}_r(\mathbf{t}) = [X_r(t_{r1}),..., X_r(t_{rp_r})]', \\
& \bm{\mu}(\mathbf{t}) = [\mu(t_{r1}),..., \mu(t_{rp_r})]',
\end{split}
\end{equation*}
\begin{equation*}
\bm{\Phi}(\mathbf{t}) = \begin{bmatrix} \phi_1(t_{r1}) & \dots & \phi_K(t_{r1}) \\ 
                                  \vdots       & \ddots & \vdots  \\
                                  \phi_1(t_{rp_r}) & \dots & \phi_K(t_{rp_r})  \\
                                  \end{bmatrix}.
\end{equation*}
\label{prop:1}
\end{proposition}

\begin{proof}
See Section A.2 in supplementary document.
\end{proof}

Based on the updated FPC scores for in-service unir $r$, the posterior predicted mean $\hat{X}^{*}_r(\tilde{t})$, of $X_r(t)$ for any future time point $\tilde{t}\geq t^{\ast}$ where $\tilde{t} \in \mathcal{T} $ is given as 

\begin{equation*}
\hat{X}^{*}_r(\tilde{t}) =  \mu(\tilde{t}) + \sum_{k=1}^{K}\xi^{*}_{rk}\phi_{k}(\tilde{t}).
\end{equation*}
Similarly, the posterior variance $\big(\hat{\sigma}^{*}_r(\tilde{t})\big)^2$ can be computed as
\begin{align*}
\big(&\hat{\sigma}^{*}_r(\tilde{t})\big)^2 =\\
&\qquad \sigma^2_\mu(t)+ \sum_{k_1=1}^{K}\sum_{k_2=1}^{K}\big(\bm{\Sigma}^{*}\big)_{k_1k_2}\phi_{k_1}(\tilde{t})\phi_{k_2}(\tilde{t}) + \sigma_\epsilon^2, 
\end{align*}
where $\big(\bm{\Sigma}^{*}\big)_{k_1k_2}$ indicates the $(k_1, k_2)$th element of the covariance matrix $\bm{\Sigma}^{*}$.


Despite our focus on the target stream $s$ we note that our framework can predict every individual stream for the in-service unit $r$. This ability to provide individualized predictions is a key feature of the proposed methodology compared the data fusion literature that predicts a single aggregated signal. Further, one differentiating factor is that we allow irregularly sampled data from each stream where time points of each signal $\lbrace t^{(l)}_{iu}: u=1,...,p^{(l)}_i\rbrace \subset \mathcal{T}$ do not need to be identical or regularly spaced across streams. Indeed, such situations are quite common in practice because most multi-stream data is gathered from different types of sensors. Therefore, the proposed approach is applicable to a wide array of practical situations.  



\section{Numerical Case Study}
\label{S:5}
\subsection{General Settings}

In this section, we discuss the general settings used to assess the proposed model, denoted as FPCA-GP. We evaluate the model by performing experiments with both synthetic and real-world data. We report the prediction accuracy at varying time points $t^{\ast}$ for the partially observed unit $r$. Specifically, for the time domain $\mathcal{T} = [a,b]$, we assume that the on-line signals from the in-service unit $r$ are partially observed in the range of $[a, t^{\ast}=a+\gamma(b-a)]$, referred as $\gamma$-observation. We set $\gamma = 25\%, 50\%$, and $75\%$ for every case study. In the extrapolation interval $[t^{\ast},b]$, we use the mean absolute error (MAE) between the true signal value $X^T_r(t)$ and its predicted value $\hat{X}^{(s)}_r(t)$ at $U$ evenly spaced test points (denoted as $t_{u}$ for $u=1,...,U$) as the criterion to evaluate our prediction accuracy.

\begin{equation}
\text{MAE} = \frac{1}{U}\sum_{u=1}^{U}\big|\hat{X}^{(s)}_r(t_u) - X^T_r(t_u) \big|, \ t_u\in [t^{\ast},b] .
\label{eq:5.1}
\end{equation}

We report the distribution of the errors across $G$ repetitions using a group of boxplots representing the MAE for the testing unit $r$ at diffrent $\gamma$-observation percentiles. Further, we benchmark our method with two other reference methods for comparison: (i) The FPCA approach for single stream settings denoted as FPCA-B. In this method we only consider the target stream $s$ \cite{zhou2011degradation, kontar2018nonparametric}. Note that we incorporate our Bayesain updating procedure to update predictions as new data is observed. (ii) The Bayesian mixed effect model with a general polynomial function whose degree is determined through an Akaike information criteria (AIC) \cite{rizopoulos2011dynamic,son2013evaluation, kontar2017remaining}. We denote this methods as ME. The ME model intrinsically applies a Bayesian updating scheme as more data is obtained from the in-service unit. Detailed codes for both reference methods are included in our supplementary materials.

\subsection{Numerical Study with Synthetic Data}
First, we show the numerical results of the proposed model performed on synthetic data. For this experiment, we assume that two streams $l\in \lbrace 1, 2\rbrace$ of data are observed from two different sensors embedded in each unit. The target stream of interest is $l=1$. To generate signals possessing heterogeneity, we suppose there are two separating environments, denoted by environment I and II. We generate signals for each unit using different underlying functions depending on which environment the unit is in. This is illustrated in Figure \ref{fig:first}. As shown in the figure, the underlying trend of the target stream ($l=1$) will vary under different profiles of stream $l=2$. To relate this setting with real-world application, consider $l=1$ as the degradation level and $l=2$ as the temperature profile. Then from Figure \ref{fig:first}, we have that units operating under different temperature profiles will exhibit different trends.

\begin{figure}[t!]
\begin{center}
\centerline{\includegraphics[width=\columnwidth]{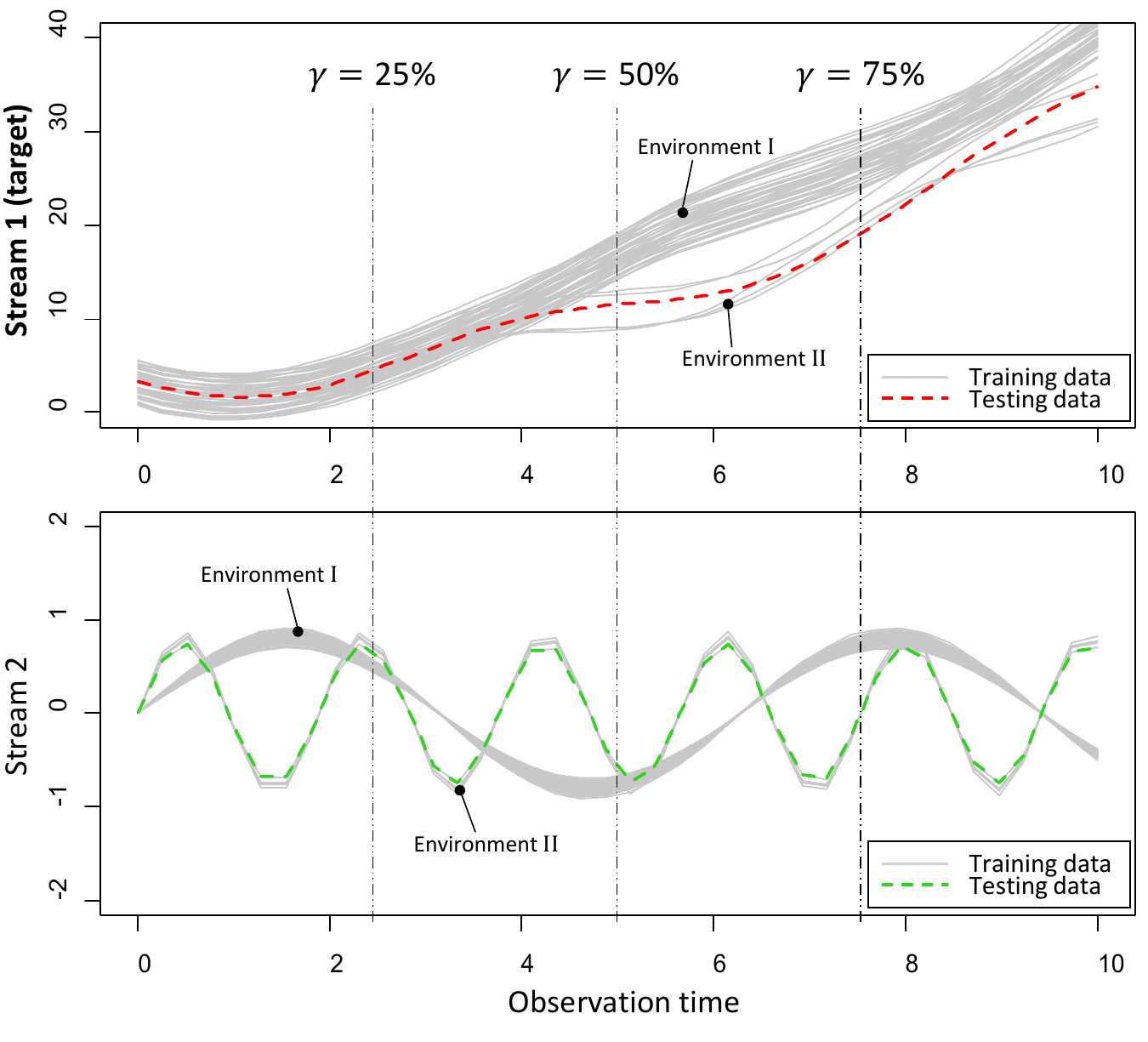}}
\caption{Illustration of generated true curves (90\% heterogeneity case).}
\label{fig:first}
\end{center}
\vskip -0.2in
\end{figure}

\begin{figure*}[t!]
\begin{center}
\centerline{\includegraphics[height=3.1in]{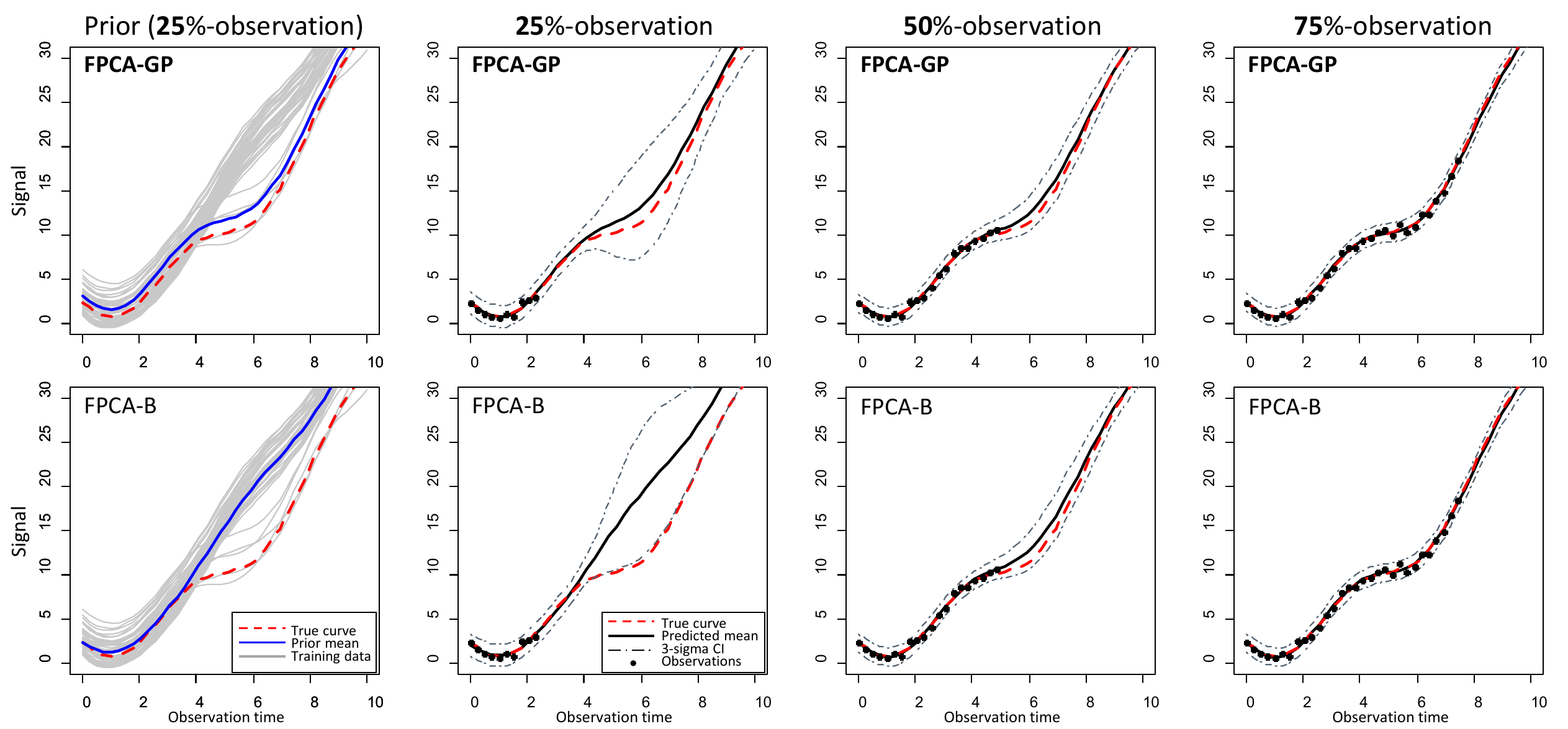}}
\caption{Illustration of the FPCA-GP and FPCA-B prediction (90\% heterogeneity case). The first column illustrates the respective prior mean of FPCA-GP and FPCA-B before updating in the case of 25\%-observation.}
\label{fig:second}
\end{center}
\vskip -0.2in
\end{figure*}

We generate a training set of $N=50$ units and one testing unit $r$ whose signals are partially observed. Also we repeat the experiment $G=100$ times. Historical units are operated in either environment I or II whereas the in-service unit is operated in environment II. The population of historical units is created according to three levels of heterogeneity: (i) \textit{0\% heterogeneity} where  all units in the historical database are operated under environment II (similar to that of the testing unit) (ii) \textit{50\% heterogeneity} where 25 units are distributed to each environment (iii) \textit{90\% heterogeneity} where only 5 units are assigned to environment II. Conducting the experiments across a homogeneous setting and a heterogeneous setting, where the in-service unit belongs to the minority group with only 10\% ratio, will allow us to investigate  the robustness of our approach.

For units in environment I, the signals from respective streams $l=1,2$ are generated according to $X_i^{(1)}(t) = 0.3t^2 - 2\sin{(w_{1,i}\pi t)} + w^{\text{I}}_{2,i}$ and $X_i^{(2)}(t)=2w_{1,i}\sin t$, where $w_{1,i}\sim\mathcal{N}(0.4, 0.03^2)$ and $w^{\text{I}}_{2,i}\sim \mathcal{U}(0, 5)$, where $\mathcal{U}$ denotes the uniform distribution. For units in environment II, we generate the signals as $X_i^{(1)}(t) = 0.3t^2 - 2\sin{(w_{1,i}\pi t^{0.85})} + 3(\arctan(t-5) + \frac{\pi}{2})+ w^{\text{II}}_{2,i}$ and $X_i^{(2)}(t)=2w_{1,i}\sin 0.3t$ where $w^{\text{II}}_{2,i}\sim \mathcal{U}(1.5, 6.5)$. Measurement error $\epsilon \sim \mathcal{N}(0,0.05^2)$ is assumed similar across both streams.

Figure \ref{fig:first} illustrates training signals $X^{(1)}_i(t)$ in the case of 90\% heterogeneity. It is crucial to note that at early stages (ex: $\gamma = 25\%$), it is hard to distinguish between the two different trends in stream $l=1$. We model that on purpose to check if our model can leverage information from stream $l=2$ to uncover the underlying heterogeneity at early stages. \textit{This in fact is a common feature in many health related applications, as many diseases remain dormant at early stages and it is only through measuring other factors we can predict there evolution early on.}

\begin{figure}[b!]
\begin{center}
\centerline{\includegraphics[width=\columnwidth]{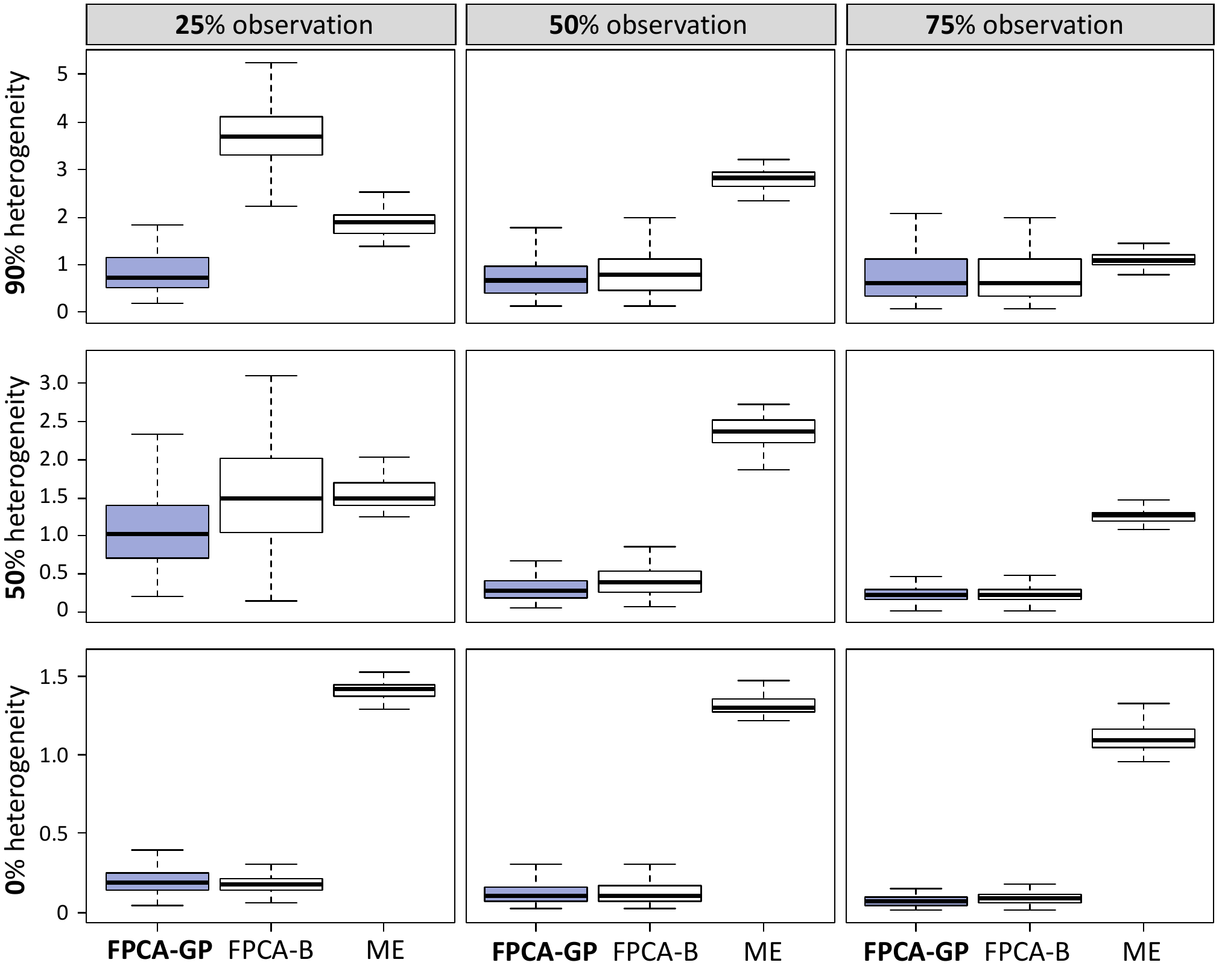}}
\caption{Box plots of MAEs for comparative models for synthetic data.}
\label{fig:boxplot}
\end{center}
\end{figure}

The results are illustrated in both Figure \ref{fig:second} and \ref{fig:boxplot}. Based on the figures we can obtain some important insights. First, the FPCA-GP clearly outperforms the FPCA-B. This is specifically obvious at early stages ($\gamma = 25\%$) and when the data exhibits heterogeneity (90\% and 50\% heterogeneity). This confirms the ability of our model to borrow strength from information across different streams to discern the heterogeneity and enhance predictive accuracy at early stages. This result is very motivating specifically since at  $\gamma = 25\%$ data from in-service unit $r$ is sparse and all signals in stream $s$ have similar behaviour which makes it hard to uncover future heterogeneity. It further implies that the FPC scores of the testing unit are appropriately estimated by the proposed approach, as shown in the first column of Figure \ref{fig:second}. From the figure, we observe that the estimated prior mean from the FPCA-GP appropriately follows the signals in environment II, whereas the prior mean from the FPCA-B follows the signals in environment I, which is the majority. Second, as expected, prediction errors significantly decrease as the percentiles increase. Thus, our Bayesian updating framework is able to efficiently utilize new collected data and provide more accurate predictions as $t^{\ast}$ increases. Third, the results show that ME behaved the worst and its predictions accuracy merely decreases at later stages. This result illustrates the vulnerabilities of parametric modeling and demonstrates the ability of our non-parametric modeling to avoid model misspecifications. Fourth, the results confirm that even in the case where other streams have no effect on the target stream (\textit{0\% heterogeneity}) the FPCA-GP is competitive compared to FPCA-B. This highlights the robustness of the FPCA-GP.

\subsection{Numerical Study with Real-world Data}
In this section, we discuss the numerical study using real-world data provided by the National Aeronautics and Space Administration (NASA). The dataset contains degradation signals collected from multiple sensors on an aircraft turbofan engine. This dataset was generated from a simulation model, developed in Matlab Simulink, called commercial modular aero-propulsion system simulation (C-MAPSS). This system simulates degradation signals from multiple-sensors, installed in several components of an aero turbofan engine, under a variety of environmental conditions. The list of the components includes Fan, LPC, HPC, and LPT, and are illustrated in Figure \ref{fig:NASA}. Refer to \citet{saxena2008damage} for more details about turbofan engine data. The dataset is available at \citet{saxena2008PHM08}. The dataset is composed of 21 sensor streams from 100 training and 100 testing units. Following the analysis of \citet{Liu2013Data}, we select the 11 most crucial streams. Some signals from these streams are shown in Figure \ref{fig:NASAtrend}. We provide the detailed list and description of sensors in the supplementary materials. In our analysis we truncate the time range $(0,100]$ and predict the testing signal over the time range $(100, 160]$.

\begin{figure}[t!]
\begin{center}
\centerline{\includegraphics[width=\columnwidth]{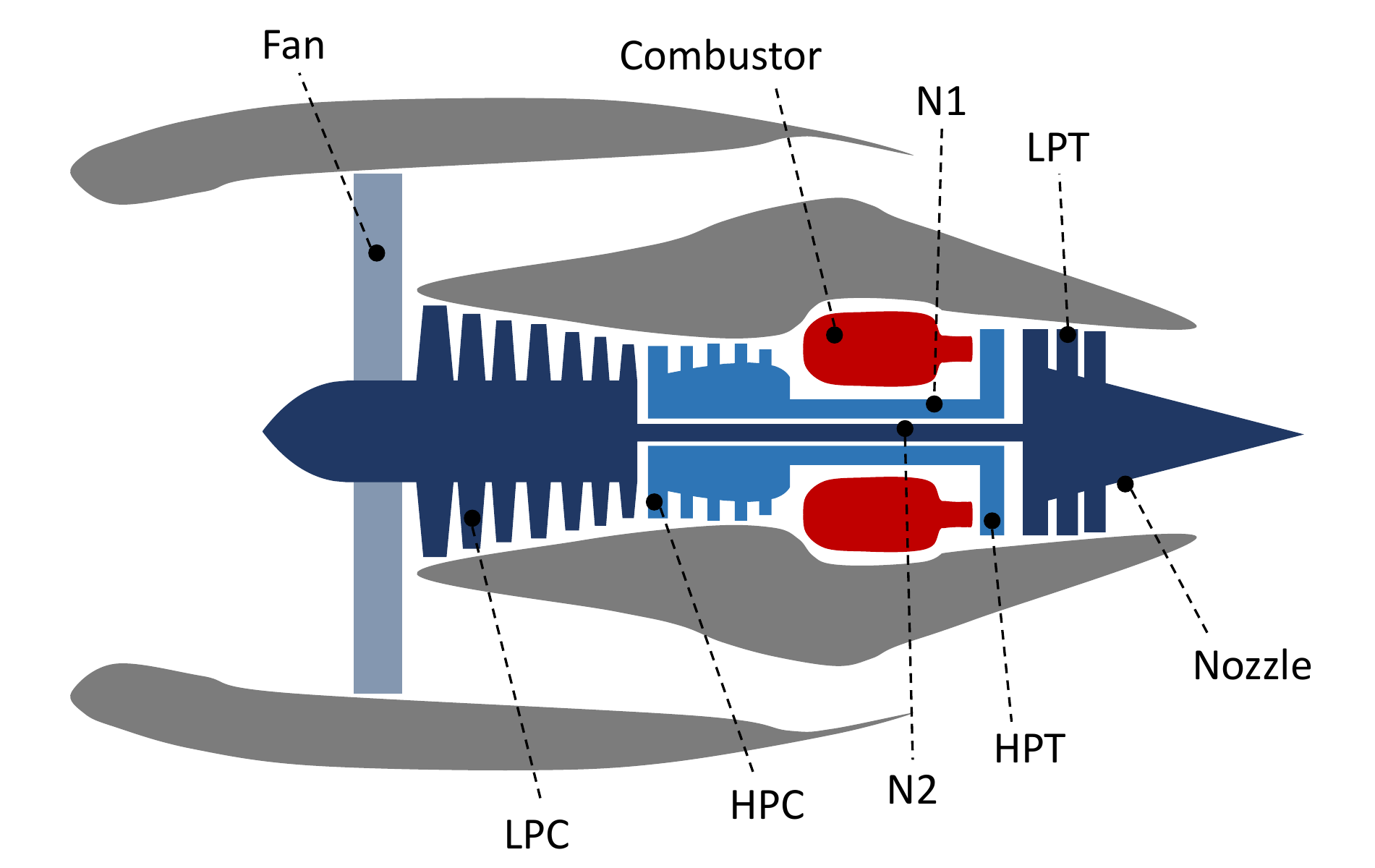}}
\caption{A schematic diagram of turbofan engine \cite{liu2012user}.}
\label{fig:NASA}
\end{center}
\vskip -0.2in
\end{figure}

\begin{figure}[t!]
\begin{center}
\centerline{\includegraphics[width=\columnwidth]{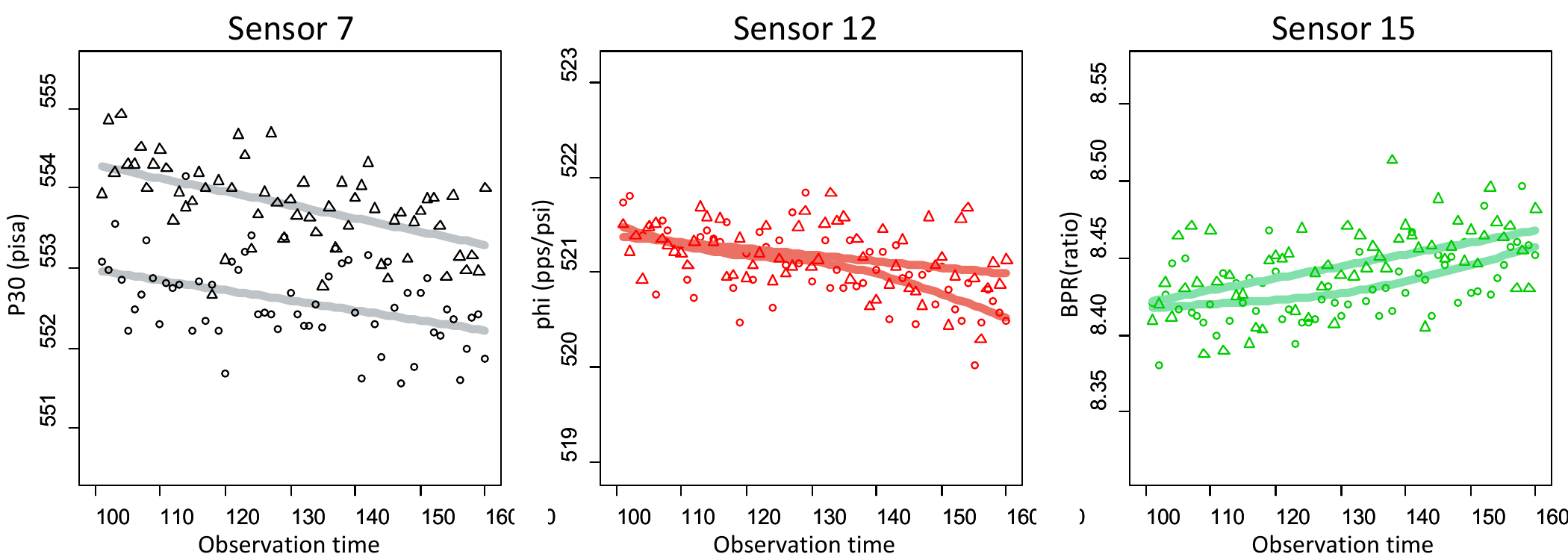}}
\caption{Selective examples for degradation signals from
turbofan engine data.}
\label{fig:NASAtrend}
\end{center}
\vskip -0.2in
\end{figure}

Table \ref{tab:result} demonstrates that the MAE results of stream 4 and 15. Note that these two streams have shown to have the largest impact on failure \cite{fang2017multistream} and therefore, due to space limitation, we only focus on their predictive results. Results for other signals are provided in the supplementary materials. Note that we include the standard deviation of MAE across the testing units.

The results clearly show that our approach is far more superior than benchmarks for the real-world data. For all provided cases, the mean of MAE for the FPCA-GP is less than that of the FPCA. Once again this highlights the importance of leveraging information from all streams of the data. Another important insight from this study is that our model was able to outperform the ME even though the curves from Figure \ref{fig:NASAtrend} seem to exhibit a clear parametric trend. This further highlights the robustness of our method and its ability to safeguard against parametric misspecifications.

\begin{table}[t!]
\caption{Mean and standard deviation (STD.) of comparative models performed on the NASA dataset. The values for the sensor 15 are scaled by $\times 10^{-2}$ . F-GP and F-B indicates FPCA-FP and FPCA-B, respectively. The best result in each case is boldfaced.}
\label{sample-table}
\vskip 0.15in
\begin{center}
\begin{small}
\begin{sc}
\begin{tabular}{p{8.4mm}p{7.5mm}p{7.5mm}p{7.5mm}p{7.5mm}p{7.5mm}p{7.5mm}}
\toprule
 & \multicolumn{3}{c}{Sensor 4} & \multicolumn{3}{c}{Sensor 15 ($\times 10^{-2}$)} \\
\midrule
Model & 25\% & 50\% & 75\% & 25\% & 50\% & 75\% \\
\midrule
F-GP   & \textbf{3.26}  & \textbf{3.21} & \textbf{3.19}     & \textbf{1.62}& \textbf{1.62} & \textbf{1.57} \\
(std.) & (0.36)   & (0.42)  & (0.48)      &(0.18)& (0.19) &(0.33)\\
F-B    & 3.49  & 3.37  & 3.31    &1.76  &1.75 &1.63\\
(std.) & (0.45)   & (0.56)  & (0.54)          & (0.28)  &(0.31) & (0.36)\\
ME     & 3.51   & 3.38  & 3.34  &1.79&1.77 &1.65\\
(std.) & (0.45)   & (0.59)  & (0.55)  &(0.28)&(0.32)&(0.37)\\
\bottomrule
\end{tabular}
\end{sc}
\end{small}
\end{center}\label{tab:result}
\vskip -0.1in
\end{table}

\section{Discussion}
\label{S:6}
In this study, we developed a non-parametric statistical model that can extrapolate individual signals in a multi-stream data setting. Using both synthetic and real-world data, we demonstrate our models ability to borrow strength across all streams of data, predict individual streams, account for heterogeneity and provide accurate real-time predictions where an empirical Bayesian approach updates our predictor as new data is observed in real-time. Since we work in the regime of streaming data, the frequency with which we receive data is very high. Due to this, our model needs to be efficient in terms of the time taken to make each update. With the assumption that all signals from $M$ streams have $Q$ observations, the complexity of multivariate FPCA for multi-stream data is $\mathcal{O}(M^2NQ^2 + M^3Q^3)$ \cite{fang2017multistream}. In our model, the computationally expensive steps are the FPCA for the target stream (Section \ref{S:4.1}) and the implementation of GP for estimating the FPC scores (Section \ref{S:4.2}). Following \citet{xiao2016fast}, the complexity of the former is $\mathcal{O}(QN^2+N^3)$. While complexity of a GP with an $N\times N$ covariance matrix is $\mathcal{O}(N^3)$ \cite{Rasmussen1st}. Given that we implement $K$ independent GP models the complexity of estimating the FPC scores is $\mathcal{O}(KN^3)$. Combing the above observations, we conclude that the complexity of our procedure is $\mathcal{O}(QN^2+N^3 + KN^3)$. Typically, we have that $M,N, K\ll Q$, also, in real-time $Q$ is increasing.  Thus, our model is clearly more efficient than multivariate the FPCA and applicable in practice in a real-time streaming environment.  

\section{Software and Data}
Technical proofs, the used dataset, a detailed code and additional numerical results are available in the \href{https://alkontar.engin.umich.edu/publications/}{supplementary materials}.

\bibliography{figs/example_paper}
\bibliographystyle{icml2019}

\end{document}